\documentclass[conference]{IEEEtran}
\IEEEoverridecommandlockouts
\usepackage{cite}
\usepackage{amsmath,amssymb,amsfonts,amsthm,mathrsfs}
\usepackage{graphicx}
\usepackage{textcomp}
\usepackage[dvipsnames]{xcolor}
\def\BibTeX{{\rm B\kern-.05em{\sc i\kern-.025em b}\kern-.08em
    T\kern-.1667em\lower.7ex\hbox{E}\kern-.125emX}}

\usepackage{slashbox}

\usepackage{dsfont}
\usepackage{hyperref}

\usepackage{tikz}
\usepackage{pgfplots}
\usetikzlibrary{arrows,arrows.meta,calc,quotes,angles,shapes,plotmarks}

\usepackage{algorithm}
\usepackage{algpseudocode}
\algnewcommand\algorithmicforeach{\textbf{for each}}
\algdef{S}[FOR]{ForEach}[1]{\algorithmicforeach\ #1\ \algorithmicdo}

\newtheorem{theorem}{Theorem}[section]

\theoremstyle{remark}
\newtheorem{remark}[theorem]{Remark}

\numberwithin{equation}{section}

\newcommand{\indep}{\rotatebox[origin=c]{90}{$\models$}}

\begin{document}

\title{Policy Gradient Optimal Correlation Search for Variance Reduction in Monte Carlo simulation and Maximum Optimal Transport\\
}


\author{\IEEEauthorblockN{Pierre Bras}
\IEEEauthorblockA{\textit{Laboratoire de Probabilités, Statistique et Modélisation} \\
\textit{Sorbonne Université}\\
Paris, France \\
pierre.bras@sorbonne-universite.fr}
\and
\IEEEauthorblockN{Gilles Pagès}
\IEEEauthorblockA{\textit{Laboratoire de Probabilités, Statistique et Modélisation} \\
\textit{Sorbonne Université}\\
Paris, France \\
gilles.pages@sorbonne-universite.fr}
}

\maketitle

\begin{abstract}
We propose a new algorithm for variance reduction when estimating $f(X_T)$ where $X$ is the solution to some stochastic differential equation and $f$ is a test function. The new estimator is $(f(X^1_T) + f(X^2_T))/2$, where $X^1$ and $X^2$ have same marginal law as $X$ but are pathwise correlated so that to reduce the variance. The optimal correlation function $\rho$ is approximated by a deep neural network and is calibrated along the trajectories of $(X^1, X^2)$ by policy gradient and reinforcement learning techniques. Finding an optimal coupling given marginal laws has links with maximum optimal transport.
\end{abstract}


\section{Introduction}

In this paper we focus on the estimation of $f(X_T)$ by Monte Carlo methods, where $X$ is the solution of the Stochastic Differential Equation (SDE):
\begin{equation}
\label{eq:def_X}
X_0 \in \mathbb{R}^d, \quad dX_t = b(X_t) dt + \sigma(X_t) dW_t,
\end{equation}
where $W$ is a Brownian motion.
We propose the new estimator
\begin{equation}
\label{eq:new_estimator}
\frac{1}{2}\left( f(X^1_T) + f(X^2_T) \right),
\end{equation}
where $X^1$ and $X^2$ have the same marginal law which as $X$ but are correlated as follows:
\begin{align}
\label{eq:X1t}
& dX^1_t = b(X^1_t) dt + \sigma(X^1_t) dW^1_t, \\
\label{eq:X2t}
& dX^2_t = b(X^2_t) dt + \sigma(X^2_t) dW^2_t,
\end{align}
where $W^1$ and $W^2$ are two Brownian motion correlated by some adapted process $\rho$ such that
\begin{equation}
W^2_t = \int_0^t \rho_u dW^1_u + \int_0^t (I - \rho_u \rho_u^\top)^{1/2} dW^3_u 
\end{equation}
where $I$ stands for the identity matrix and $W^3$ is a Brownian motion independent from $W^1$. This introduces a new class of control varietes for SDEs. The solutions to \eqref{eq:X1t} and \eqref{eq:X2t} are then approximated by an Euler-Maruyama numerical scheme and the fact that $X^1$ and $X^2$ are constructed so that they have the same marginal law as $X$ guarantees that the bias of our estimator is the same as the vanilla estimator and is equal to the bias of the discretized scheme.
Considering the constant correlation $\rho_t=-I$ boils down to $W^1=-W^2$ giving the classic antithetic scheme and has proved its efficiency for variance reduction, see \cite{hammersley1956}, \cite[Chapter 3.1.2]{pages2018} as well as Theorem \ref{thm:antithetic} in the Appendix.
The correlation process $\rho$ is calibrated such that to minimize the resulting variance of the new estimator \eqref{eq:new_estimator}, which is done by policy gradient and reinforcement learning techniques.

\smallskip

Reinforcement Learning (RL) studies how an intelligent agent should take actions in an environment in order to maximize its cumulative reward. In the case of Markov Decision Processes (MDP), an agent chooses its next action $a_t$ according to its policy $\pi$ depending on its current state $s_t$; it then gets into a next state $s_{t+1}$ depending on the response of the environment and gets the instantaneous reward $r_{t}$, giving the equations
\begin{align}
a_{t} = \pi(s_t), \quad s_{t+1} = \mathscr{S}(s_t, a_{t}), \quad r_{t} = \mathscr{R}(s_t, a_{t}) .
\end{align}
The functions $\mathscr{S}$, $\mathscr{R}$ and $\pi$ may have random components.
The objective of the agent is to maximize its expected total reward:
\begin{equation}
\textstyle \max_{\pi} \ \mathbb{E}\big[ \sum_{t=0}^T \gamma^{t} r_t \big]
\end{equation}
where $\gamma \in (0,1]$ is the discount factor.

Policy Gradient algorithm, in its simplified version, consists in parametrizing the policy of the agent by some finite-dimensional parameter $\theta$ -- in general the weights of a neural network -- and updated with the gradient of the reward; the gradient is tracked all along the trajectory of the process \cite{giles2005,giles2007} or by telescopic increments (see \eqref{eq:reward} in the following).
As we use it in our setting, it requires to know the expression of the environment response and reward functions $\mathscr{S}$ and $\mathscr{R}$, which is the case in our SDE setting where the coefficients $b$ and $\sigma$ of \eqref{eq:def_X} are assumed to be known, but which may not be always the case for general RL.
We also require a large batch size of the training trajectory for policy gradient.
Policy gradient was used for SDE-based stochastic control problems \cite{gobet2005,han2016,wang2019,buehler2019,carmona2021,hure2021,lauriere2021,bachouch2022} along with more traditional methods such as Forward-Backward SDEs \cite{peng1999}, Hamilton-Jacobi-Bellman optimality conditions \cite{bellman2010} through partial differential equations and dynamic programming \cite{kushner2001}.

Other and more sophisticated reinforcement learning algorithms generally use two or more combined neural networks updated by gradient descent, the actor network parametrizing the policy of the agent and the critic network approximating the value of Q-function of the expected future rewards \cite{rl_an_intro,q_learning,atari}:
\begin{equation}
Q(s,a) = \max_\pi \ \mathbb{E}\big[ \sum_{t=t_0}^T \gamma^{t-t_0} r_t \big| \ s_{t_0}=s, a_{t_0}=a \big].
\end{equation}
We also refer to \cite{a3c,ppo,td3}.

\smallskip

Variance reduction for Monte Carlo simulation of SDEs can be achieved by adding a zero-mean but correlated variable to the estimator (see \cite[Chapter 3]{pages2018} and \cite{boyaval2009}) or by changing the coefficients of the SDE while keeping the same marginal law or expectation of the final value \cite{newton1994}.
In \cite{lemaire2010}, the authors add a pre-processing step for importance sampling of a normal-distributed random vector, calibrated by gradient decent.
More recently, \cite{hinds2022} extends the method from \cite{milstein2009} and uses neural instead of linearly regressed control variates for some classes of SDEs.

The idea of reparametrizing the Brownian motion through rotations matrices was for example used in \cite{cruzeiro2004} for reducing bias in this case, where the original Brownian motion is replaced by a rotated Brownian motion thus having same marginal law.

The problem of minimizing a variance of two coupled random vectors with given marginal laws has links with optimal transport and more precisely maximum optimal transport, where we try to maximize the $L^2$-distance between these two random vectors instead of minimizing it \cite{peyre2019}. To our knowledge this is a new application of maximum optimal transport; we refer to the Remark \ref{rk:MOT}.

%
%
%
%
%

\smallskip

In this paper we implement this variance reduction method as an agent-environment problem and we give a practical way of parametrizing the correlation matrix $\rho_t$; for computational purposes we restrict ourselves to some sub-class of matrices satisfying the condition $\rho \rho^\top \le I$.
We give numerical experiments for option pricing on some mathematical finance models (Black-Scholes, Heston) and in particular we give examples of non-trivial solutions with non-constant correlation $\rho_t$ that reduce the variance of the estimator.
For the calibration algorithm we focus on policy gradient algorithms but also implement and try more RL-oriented methods, with less success for now however, see Section \ref{sec:rl}.
We give the implementation as the \texttt{Python} package \texttt{relocor}: REinforcement Learning Optimal CORrelation search and with a demonstration notebook available at \textcolor{blue}{\url{https://github.com/Bras-P/relocor}}, where the variance reduction problem is written as an OpenAI \texttt{gym} (or \texttt{gymnasium}) environment \cite{gym}.


\smallskip

\textbf{Notations:} for $x \in \mathbb{R}$ we denote $x_+ := \max(x,0)$. For $a, b \in \mathbb{N}$ we denote $\mathcal{M}_{a,b}(\mathbb{R})$ the set of $a \times b$ real-valued matrices; if $a=b$ we also write $\mathcal{M}_a(\mathbb{R}) = \mathcal{M}_{a,a}(\mathbb{R})$.

\section{Setting and main Algorithms}

\subsection{Stochastic setting}

Let us consider the following $d_1$-dimensional SDE:
\begin{equation}
X_0 \in \mathbb{R}^d, \quad dX_t = b(X_t) dt + \sigma(X_t) dW_t,
\end{equation}
where $b : \mathbb{R}^{d_1} \to \mathbb{R}^{d_1}$, $\sigma : \mathbb{R}^{d_1} \to \mathcal{M}_{d_1,d_2}(\mathbb{R})$ and $W$ is a $\mathbb{R}^{d_2}$-valued standard Brownian motion with filtration $(\mathcal{F}_t)$. Let $f: \mathbb{R}^{d_1} \to \mathbb{R}$ be a test function and $T>0$ a finite time horizon.

We consider the following two processes:
\begin{align}
\label{eq:def_X1}
& dX^1_t = b(X^1_t) dt + \sigma(X^1_t) dW^1_t, \\
\label{eq:def_X2}
& dX^2_t = b(X^2_t) dt + \sigma(X^2_t) dW^2_t,
\end{align}
where $W^1$ and $W^2$ are two correlated Brownian motions. More precisely, for $W^3$ some Brownian motion independent of $W^1$ we write
\begin{align}
\label{eq:def_W2}
W^2_t = \int_0^t \rho_s W^1_s ds + \int_0^s \big( I_{d_2} - \rho_s \rho_s^{\top} \big)^{1/2} W^3_s
\end{align}
where $\rho$ is some $\mathcal{F}$-measurable process with values in $\mathcal{M}_{d_2}(\mathbb{R})$ such that $I_{d_2} - \rho_s \rho_s^{\top}$ is a non-negative matrix.
Then Levy's characterization guarantees that $W^2$ is still a Brownian motion, and then $X^1$ and $X^2$ have same marginal law as $X$.
We recall that every coupling between $X^1$ and $X^2$ can be written as such, see Theorem \ref{thm:sde_coupling} in the Appendix when taking $b_1=b_2$ and $\sigma_1=\sigma_2$.


In order to estimate $f(X_T)$ by Monte Carlo methods, we propose the estimator
$$ Z := \frac{1}{2} \left( f(X^1_T) + f(X^2_T) \right), $$
checking that $\mathbb{E}[Z] = \mathbb{E}[f(X_T)]$. Its variance is
$$ \frac{1}{2} \left( \mathbb{E}[f(X_T)^2] + \mathbb{E}[f(X^1_T) f(X^2_T)] \right). $$

In order to reduce the variance of the estimator $Z$, we consider the following stochastic control problem:
\begin{align}
\label{eq:J_def}
\min_\rho \ \mathbb{E}[f(X^1_T) f(X^2_T)],
\end{align}
where $X^1$ and $X^2$ are defined in \eqref{eq:def_X1} and \eqref{eq:def_X2} and where $\rho$ is a $\mathcal{F}$-measurable process.
More precisely, using the Markov property of the equation we can write $\rho$ as depending only on $t$ and the current state $(X^1_t, X^2_t)$ instead of the whole previous trajectory $(X^1_u, X^2_u)_{0 \le u \le t}$:
\begin{equation}
\rho_t = \rho(X^1_t, X^2_t, t).
\end{equation}

The solution to \eqref{eq:def_X1} and \eqref{eq:def_X2} is approximated by discretization through 
the Euler-Maruyama scheme:
\begin{align}
& t_k = kT/N = kh, \\
\label{eq:X_bar_1}
& X^1_{t_{k+1}} = X^1_{t_k} + hb(X^1_{t_k}) + h^{1/2} \sigma(X^1_{t_k}) \xi^1_{k+1}, \\
& X^2_{t_{k+1}} = X^2_{t_k} + hb(X^2_{t_k}) + h^{1/2} \sigma(X^2_{t_k}) \nonumber \\
\label{eq:X_bar_2}
& \qquad \qquad \cdot \big(\rho_{t_k} \xi^1_{k+1} + (I_{d_2} - \rho_{t_k}\rho_{t_k}^\top)^{1/2} \big)\xi^2_{k+1}
\end{align}
where $\xi^i_k \sim \mathcal{N}(0,I_{d_2})$, $i=1,2$, $k=1,\ldots, N$ and are mutually independent.

\subsection{Agent-Environment setting}

Let us define the environment associated to the SDEs.
The environment state is $s_t = (X^1_t, X^2_t, t)$.
We consider the action of the agent $a_t$ being the correlation matrix given by the policy $\pi$ and we parametrize this policy as a feedforward neural network with weights $\theta$:
$$ a_t = \rho_t = \pi(s_t) = \rho_\theta(s_t) .$$
For some state $s_{t_k}$ and action $\rho_{t_k}$, the next state $s_{t_{k+1}}$ is given by the equations \eqref{eq:X_bar_1} and \eqref{eq:X_bar_2}.
In order to have reward returns all along the trajectory instead of very sparse zero returns except at the last time step, we write the total reward in its telescopic expression and eliminating the constant first term:
\begin{equation}
\label{eq:reward}
R_T {=} \sum_{k=1}^N - \left( f(X^1_{t_k}) f(X^2_{t_k}) - f(X^1_{t_{k-1}}) f(X^2_{t_{k-1}}) \right) {=} \sum_{k=1}^N r_{t_k}.
\end{equation}



\subsection{Parametrization of the correlation matrix}

In order to parametrize the correlation matrix $\rho_t$ in \eqref{eq:def_W2} and satisfying the necessary condition $\rho_t \rho_t^\top \le I_{d_2}$, we use a feedforward neural network of the state input and we consider two different parametrizations of the output:
\begin{itemize}
	\item \texttt{Diag} parametrization: $\rho_t$ is a diagonal matrix with coefficients in the interval $[-1,1]$,
	\item \texttt{Ortho} parametrization: $\rho_t = BDB^\top$ where $D$ is a diagonal matrix with coefficients in $[-1,1]$ and where $B$ is a block-diagonal orthogonal matrix with $2\times 2$ rotations blocks
	$$ \begin{pmatrix}
\alpha_i & -(1-\alpha_i^2)^{1/2} \\ 
(1-\alpha_i^2)^{1/2} & \alpha_i
\end{pmatrix}  $$
for $1 \le i \le \lfloor d_2/2 \rfloor$ and where $\alpha_i \in [0,1]$ and completed by $1$ if $d_2$ is odd.
\end{itemize}
In both cases, we use sigmoid activation to enforce the constraints on the diagonal and rotation coefficients.

\subsection{Algorithms}

We consider Algorithm \ref{algo:pg} to calibrate the weights $\theta$ of the neural network $\rho_\theta$, where $\lambda$ is the step size of the gradient descent, although we also consider other optimizers than SGD such as Adam \cite{adam}. All the operations (state update, reward, gradient descent) are done in parallel on a large batch of trajectories. This algorithm thus requires a batched environment and an extension of the base \texttt{gym} environment \cite{gym}.
After each epoch or episode we may evaluate the variance on a larger batch of trajectories; this step is only for evaluation purposes and is not included in the training time.

\begin{algorithm}
\caption{Policy Gradient}\label{algo:pg}
\begin{algorithmic}
\ForEach{epoch}
	\For{$0 \le k \le N-1$ on a batch of trajectories}
		\State Choose the action $a_{t_k} = \rho_\theta(X^1_{t_k}, X^2_{t_k}, t_k)$.
		\State Update the state $s_{t_{k+1}} = \mathscr{S}(s_{t_k}, a_{t_k})$ according to \eqref{eq:X_bar_1} and \eqref{eq:X_bar_2}.
		\State Get the reward $r_{t_{k+1}} = \mathscr{R}(s_{t_k}, a_{t_k})$ according to \eqref{eq:reward}.
		\State Update the policy: $ \theta \leftarrow \theta + \lambda \nabla_{\theta} \mathscr{R}(s_{t_k}, \rho_\theta(s_{t_k})). $
	\EndFor
\EndFor
\end{algorithmic}
\end{algorithm}

The training of $\rho_\theta$ is a pre-processing step prior to the Monte Carlo simulation. In general this is faster than the corresponding online algorithm consisting in evaluating $\mathbb{E}[f(X_T)]$ at the same time as we train the agent because the training procedure converges faster than the Monte Carlo evaluation.

\section{Experiments}

In the following we will be consider the one-asset or multi-basket models:

\begin{itemize}
	\item The Black-Scholes model:
\begin{align}
\label{eq:black_scholes}
dX_t = b X_t dt + \sigma X_t dW_t 
\end{align}
where $b \in \mathcal{M}_{d_1}(\mathbb{R})$ and $\sigma \in \mathcal{M}_{d_2, d_1}(\mathbb{R})$ are constant.

	\item The Heston model: for even $d_1$, we consider $d_1'=d_1/2$ independent Heston models where the price and volatility processes read for $1 \le i \le d_1'$:
\begin{align}
\label{eq:heston}
& dS_t^i = \sqrt{V_t^i} S_t^i dB^i_t, \\
& dV_t^i = a^i(b^i - V_t^i) dt + \sigma^i \sqrt{V^i_t} dW^i_t
\end{align}
where $B^i$ and $W^i$ are standard Brownian motions with correlation $\tau^i \in [-1,1]$.

\end{itemize}

In each experiment and unless stated otherwise, for both the \texttt{Diag} and the \texttt{Ortho} parametrizations, the correlation is parametrized by a feedforward neural network with two hidden layers with 64 units each and with ReLU activation.

An synthesis of our experimental results is given in Table \ref{tab:all}.

\subsection{Simple optimal correlations for one-asset models}

We first conduct experiments on simple one-asset settings corresponding to the Black-Scholes model \eqref{eq:black_scholes} with $d_1=d_2=1$ and the Heston model \eqref{eq:heston} with $d_1'=1$, where the optimal solution to \eqref{eq:J_def} is simple for some class of payoff functions (monotonous and/or convex). We then check that the algorithms we introduced can find the optimal solution in these simple cases.

For the one-dimensional Black-Scholes model with call payoff, we empirically observe that the optimal correlation $\rho^\star$ is the simple constant solution $\rho^\star_t = -1$, which corresponds to the antithetic scheme \cite{hammersley1956}; we give a theoretical justification in a simplified case in Theorem \ref{thm:antithetic} in the Appendix.

For the Heston model with $d'_1=1$ and call payoff, we empirically observe that the optimal correlation is the constant correlation matrix $\rho^\star_t = \text{Diag}(-1,1)$
that we shall call the "minus-plus agent".

\medskip

We give the results in figure \ref{fig:bs_dim1} for the Black-Scholes model \eqref{eq:black_scholes} with payoff function $f(X_T) = (X_T - K)_+ $ and parameters
\begin{align}
\label{eq:bs_parameters}
b=0.06, \ \sigma=0.3, \ X_0 = 1, \ K = 1, \ T = 1, \ N = 50,
\end{align}
and where we plot the realized variance of the estimator, in comparison with the vanilla estimator and the antithetic estimator. The batch size for training is 512 and the batch size for evaluation at each epoch is $512 \times 16$.

\begin{figure}
\centering
\includegraphics[width=0.5\textwidth]{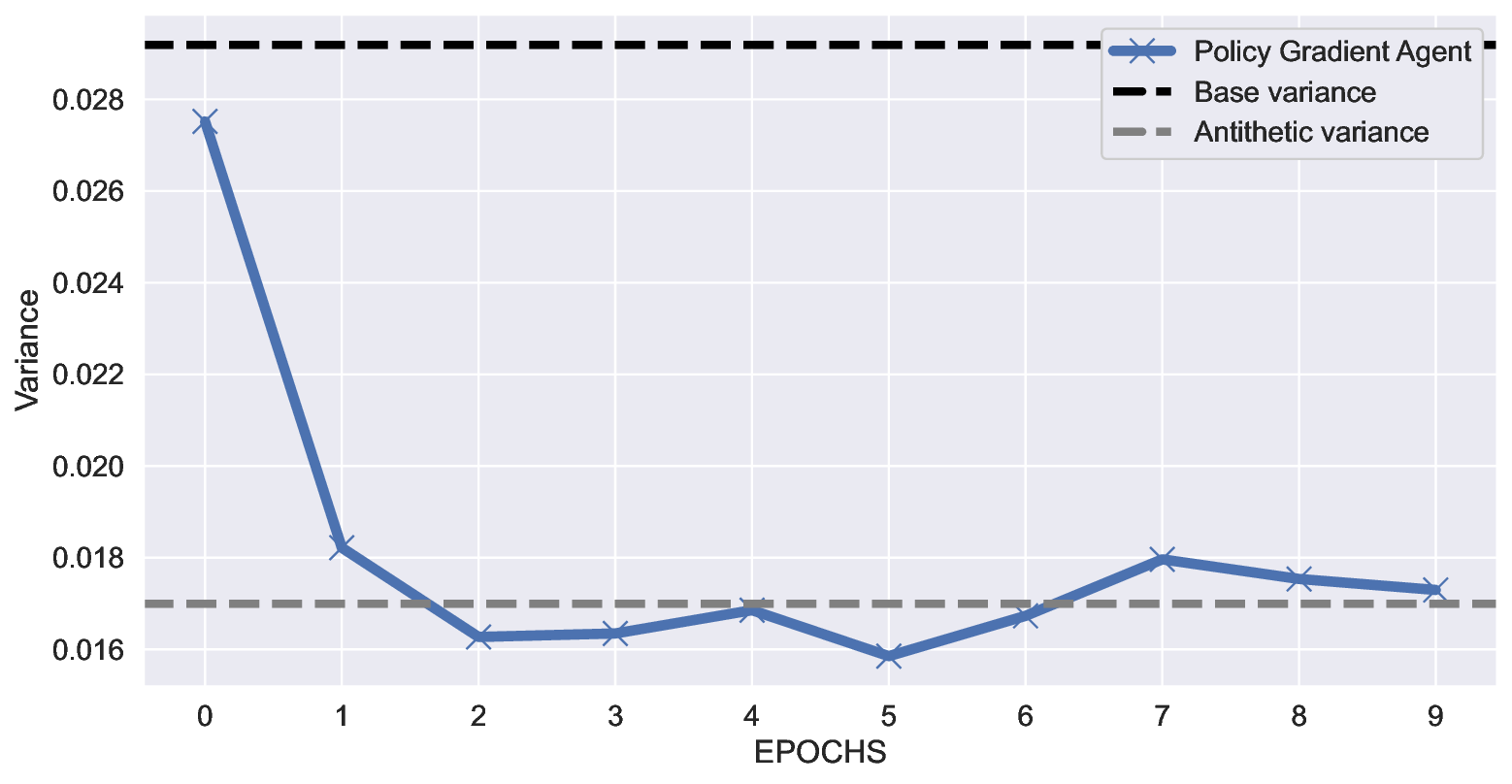}
\caption{Optimal correlation search for the Black-Scholes model with $d_1=1$.}
\label{fig:bs_dim1}
\end{figure}

\medskip

We give the results in figure \ref{fig:heston_dim1} for the Heston model \eqref{eq:heston} with payoff function $f(X_T) = (S_T-K)_+$ and parameters
\begin{align*}
& a=0.5, \ b=0.04, \ \tau = - 0.7, \ \sigma=0.5, \ S_0 = K = 1, \\
& V_0 = 1, \ T=1, \ N=50,
\end{align*}
and where we plot the realized variance of the estimator, in comparison with the vanilla estimator, the antithetic estimator and the minus-plus estimator. The batch size for training is 512 and the batch size for evaluation at each epoch is $512 \times 16$. We take the action to follow the \texttt{Diag} parametrization.

\begin{figure}
\centering
\includegraphics[width=0.5\textwidth]{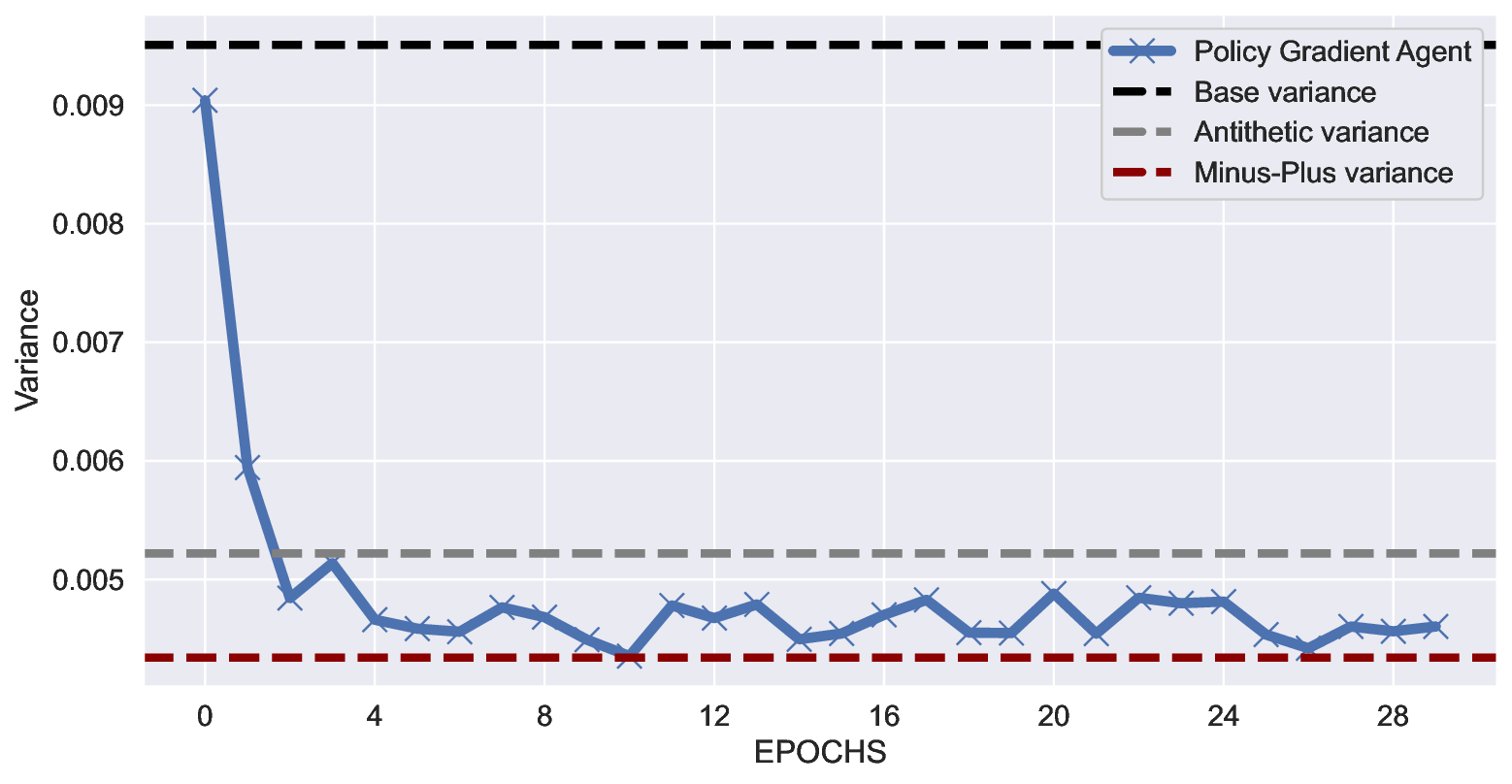}
\caption{Optimal correlation search for the Heston model with $d'_1=1$.}
\label{fig:heston_dim1}
\end{figure}

In both cases our policy gradient method can find the optimal agents and we observe that the trajectories of $\rho_t$ indeed correspond to the antithetic agent and to the minus-plus agent respectively.

\subsection{A non-trivial optimal correlation with a multi-basket model}

We give an example where the solution to \eqref{eq:J_def} is not trivial or constant and where our policy gradient estimator achieves a better variance than usual constant agents, such as antithetic and minus-plus agents.

We consider a multi-asset Black-Scholes model \eqref{eq:black_scholes} with $d_1=d_2=2$ and with the same other parameters as in \eqref{eq:bs_parameters} and payoff $f(X_T) = \langle \alpha, (X_T - K)_+ \rangle$, $\alpha, \ K \in \mathbb{R}^{d_1}$; in our case we take $K=1$ and $\alpha=1/d_1=0.5$.
We give the results in Figure \ref{fig:bs_dim2}. The batch size for training is 512 and the batch size for evaluation at each epoch is $512 \times 16$. We plot the variance for the policy gradient estimators following the \texttt{Diag} and \texttt{Ortho} parametrization respectively.
We observe that the optimal variance only corresponds to the antithetic estimator if taking the \texttt{Diag} parametrization. However if we consider the \texttt{Ortho} parametrization, we obtain non trivial solution and where the rotation is not constant; an example of trajectory under the \texttt{Ortho} parametrization is given in Figure \ref{fig:bs_dim2_trajs}.

\begin{figure}
\centering
\includegraphics[width=0.5\textwidth]{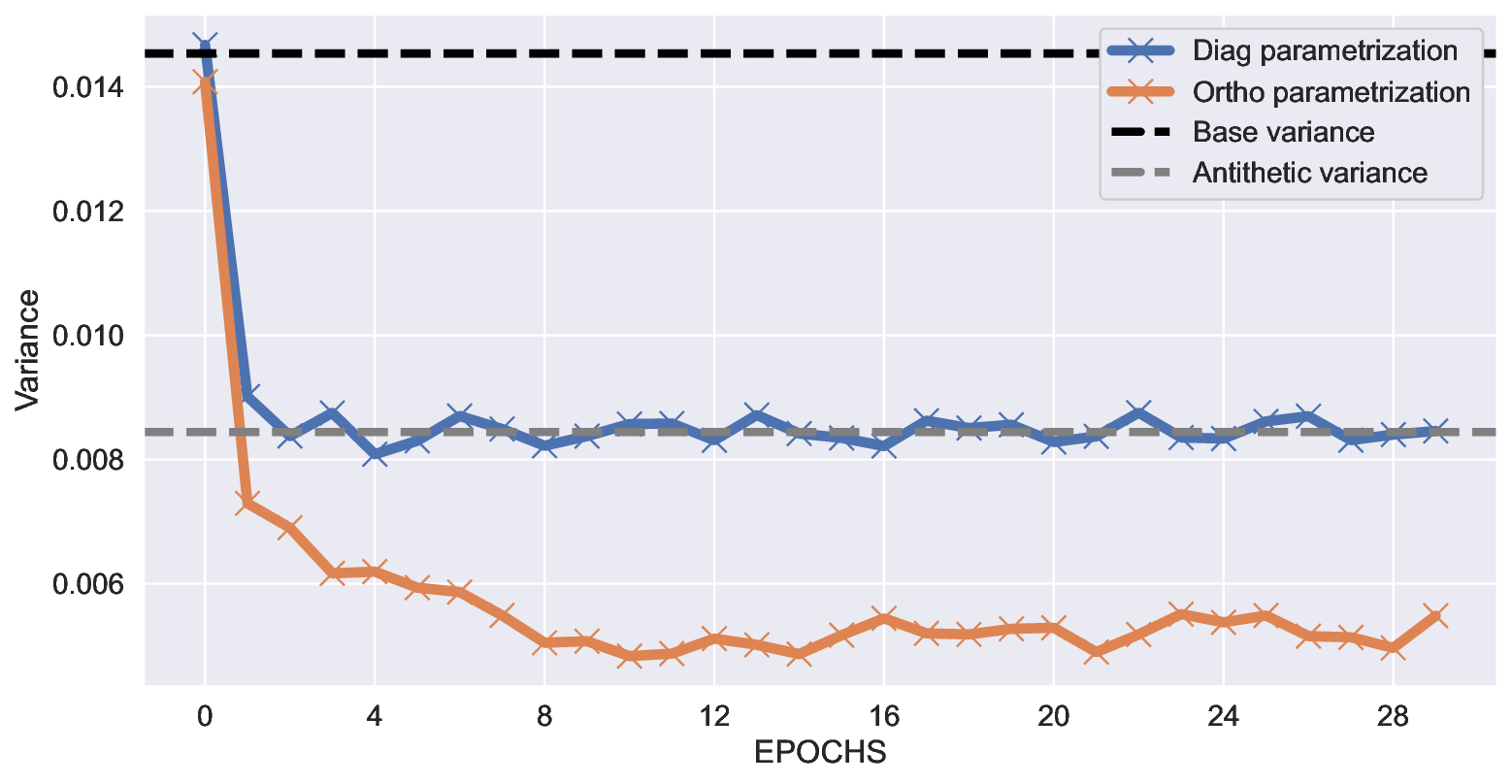}
\caption{Optimal correlation search for the Black-Scholes model with $d_1=2$.}
\label{fig:bs_dim2}
\end{figure}

\begin{figure*}
\centering
\includegraphics[width=0.9\textwidth]{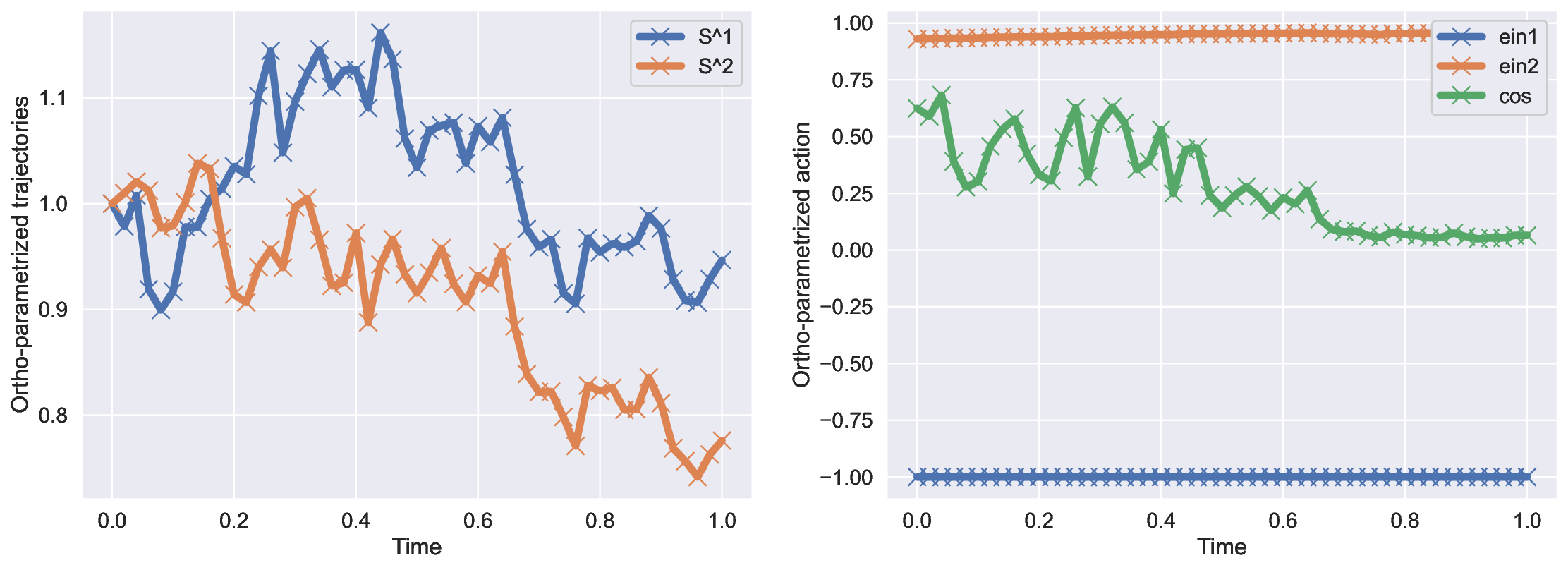}
\caption{Example of trajectory with optimal correlation parametrized by \texttt{Ortho}. On the left we plot the trajectories of the assets $S^1$ and $S^2$; on the right, \texttt{ein1} and \texttt{ein2} are the coefficients of the diagonal matrix and \texttt{cos} is the cosine of the rotation matrix.}
\label{fig:bs_dim2_trajs}
\end{figure*}

\subsection{Higher dimensional multi-basket models}

To prove the interest of our method for higher dimensional SDE models, we consider the Black-Scholes model \eqref{eq:black_scholes} with odd dimension $d_1=d_2=5$; the payoff and parameters are the same as previously.
We plot the results in Figure \ref{fig:bs_dim5} along with the value of the variance for the baseline and the antithetic estimators.

\begin{figure}
\centering
\includegraphics[width=0.45\textwidth]{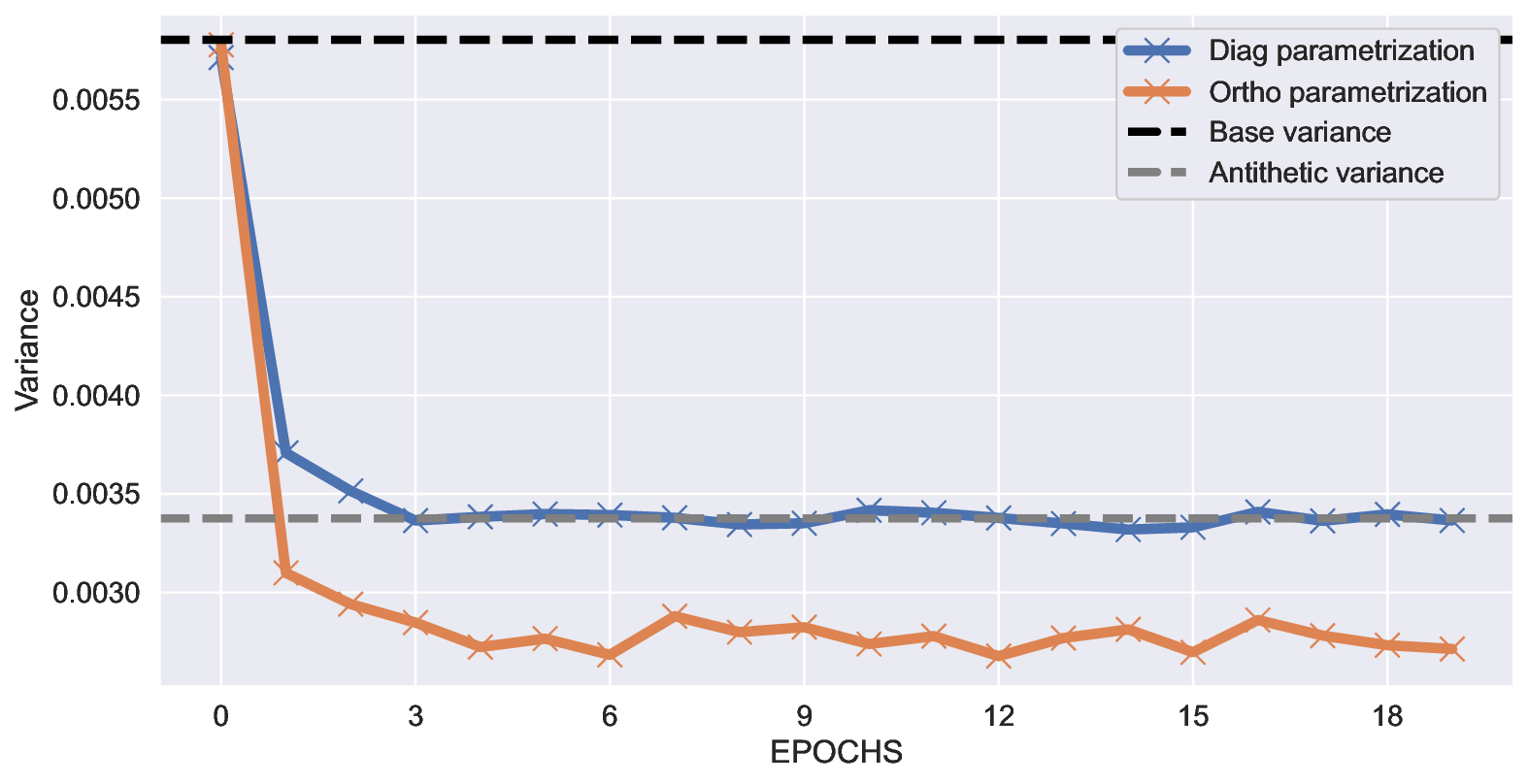}
\caption{Optimal correlation search for the Black-Scholes model with $d_1=5$.}
\label{fig:bs_dim5}
\end{figure}

We also observe that the \texttt{Diag} parametrization cannot achieve a lower variance than the antithetic estimator (corresponding to diagonal coefficients all equal to $-1$), however the \texttt{Ortho} parametrization achieves better results with non-trivial correlation.

\begin{table*}
\centering
\begin{tabular}{c|ccccc}
\hline
\backslashbox{Model}{Agent} & Baseline & Antithetic & Minus-Plus & \texttt{Diag} & \texttt{Ortho} \\ 
\hline
Black-Scholes $d_1=1$ & 2.9190 & 1.6994 & -- & 1.6994 & -- \\
Heston $d_1'=1$ & 0.9507 & 0.5220 & 0.4341 & 0.4341 & 0.4341 \\
Black-Scholes $d_1=2$ & 1.4537 & 0.84426 & 1.8756 & 0.84426 & 0.48761 \\
Black-Scholes $d_1=5$ & 0.58034 & 0.3375 & -- & 0.3375 & 0.2713 \\
\hline
\end{tabular}
\vspace*{1ex}
\caption{Variance values (multiplied by $10^2$ for readability) obtained for reference and trained estimators for the different models.}
\label{tab:all}
\end{table*}

\section{Reinforcement Learning approaches}
\label{sec:rl}

In addition to policy gradient, we tried more RL-oriented algorithms such as A2C \cite{a3c}, PPO \cite{ppo} and TD3 \cite{td3} and we use the algorithms which are already implemented and ready-to-use from Stable Baselines3 implementation \cite{stable-baselines3}.
However we could not get satisfying results with more sophisticated RL algorithms yet; indeed we observe that these last algorithms often converge to a local minimum, typically the antithetic agent for the Black-Scholes model for example, but cannot find the optimal correlation.

We believe the reasons are that SDEs models are more simple and straightforward than control problems generally tackled by RL with very sparse rewards and very long term policy-reward dependencies.
Furthermore, in our case we consider SDEs with homogeneous coefficients.
Another point is the need of a large batch size in stochastic analysis and Monte Carlo methods, as the variance essentially comes from the different possibilities of trajectories of the Brownian motion, whereas RL algorithms usually require a small batch size on the trajectories to work.

We still give our implementation as OpenIA \texttt{gym} environments that are directly usable for verification and further research. The use of these last algorithms and the links between SDE control and Reinforcement Learning remains an open research field \cite{han2016}.

\appendix

\section{Appendix}

\begin{theorem}
\label{thm:ipp_formula}
Let $g: \mathbb{R} \to \mathbb{R}$ being $\mathcal{C}^1$, $Z^1$ and $Z^3 \sim \mathcal{N}(0,1)$ with $Z^1 \indep Z^3$ and
\begin{equation}
\varphi : \rho \in [-1,1] \mapsto \mathbb{E}\big[g(Z^1) g(\rho Z^1 + \sqrt{1-\rho^2} Z^3) \big].
\end{equation}
Then we have
\begin{equation}
\label{eq:varphi_1}
\varphi'(\rho) = \mathbb{E}\big[\varphi'(Z^1) \varphi'(\rho Z^1 + \sqrt{1-\rho^2} Z^3) \big].
\end{equation}
More generally and if $g$ is $\mathcal{C}^k$ we have
\begin{equation}
\label{eq:varphi_k}
\varphi^{(k)}(\rho) = \mathbb{E}\big[\varphi^{(k)}(Z^1) \varphi^{(k)}(\rho Z^1 + \sqrt{1-\rho^2} Z^3) \big]
\end{equation}
\end{theorem}
\begin{proof}
Let $\bar{\rho} := \sqrt{1-\rho^2}$. Then
\begin{align*}
\varphi'(\rho) & = \mathbb{E}[g(Z^1) g'(\rho Z^1 + \bar{\rho}Z^3) Z^1] \\
& \quad - (\rho/\bar{\rho}) \mathbb{E}[g(Z^1) g'(\rho Z^1 + \bar{\rho}Z^3) Z^3 ].
\end{align*}
On the one side by integration by parts we have
\begin{align*}
& \mathbb{E}\big[ g(Z^1) g'(\rho Z^1 + \bar{\rho}Z^3) Z^1 \big] = \mathbb{E}\big[ g'(Z^1) g'(\rho Z^1 + \bar{\rho} Z^3) \big] \\
& \quad + \rho \mathbb{E}\big[ g(Z^1) g''(\rho Z^1 + \bar{\rho} Z^3) \big]
\end{align*}
and on the other hand
\begin{align*}
& \mathbb{E} \big[ g(Z^1) g'(\rho Z^1 + \bar{\rho} Z^3) Z^3 \big] = \mathbb{E} \big[g(Z^1) \bar{\rho} g''(\rho Z^1 + \bar{\rho} Z^3) \big]
\end{align*}
where we used the Stein formula, stating that for every differentiable function $\phi:\mathbb{R}\to \mathbb{R}$ and $Z \sim \mathcal{N}(0,1)$ we have $\mathbb{E}[\phi'(Z)] = \mathbb{E}[Z \phi(Z)]$.
Equation \eqref{eq:varphi_k} is then obtained recursively by applying \eqref{eq:varphi_1} on $g^{(k-1)}$.
\end{proof}

\begin{theorem}
\label{thm:antithetic}
Let us consider \eqref{eq:black_scholes} with $d_1=d_2=1$. Then the solution of \eqref{eq:J_def} with non-decreasing test function $f$ and while considering $\rho$ being constant, is $\rho^\star = -1$.
\end{theorem}
\begin{proof}
We recall that the solution of \eqref{eq:black_scholes} is $X_t = \exp((b-\sigma^2/2)t + \sigma W_t$.
We define
\begin{align*}
\varphi: \rho \in [-1,1] & \mapsto \mathbb{E} \big[ f(e^{(b-\sigma^2/2)T + \sigma W^1_T}) f(e^{(b-\sigma^2/2)T + \sigma W^2_T}) \big] \\
& = \mathbb{E}\big[ g(Z^1) g(\rho Z^1 + \sqrt{1-\rho^2} Z^3) \big]
\end{align*}
where $Z^1, Z^3 \sim \mathcal{N}(0,1)$ are independent and $g(x) = f \circ \exp((b-\sigma^2/2)T + \sigma \sqrt{T} x)$ is non-decreasing. Using Theorem \ref{thm:ipp_formula} we obtain
$$ \varphi(\rho) = \mathbb{E} \big[ g'(Z^1) g'(\rho Z^1 + \bar{\rho} Z^3) \big] $$
and that if $g$ is non-decreasing then $\varphi$ is non-decreasing on $[-1,1]$ and reaches its minimum in $\rho=-1$.
\end{proof}

\begin{remark}
If 
$$ \varphi'(-1) = \mathbb{E}[ g'(Z^1) g'(-Z^1) ] < 0 ,$$
for example taking $f(x)=\sigma^{-2}T^{-1}(\log(x)-\mu T)^2$ with $\mu:=b-\sigma^2/2$ giving $g(x)=x^2$ and $\varphi'(-1)=-\mathbb{E}[(Z^1)^2]$, then the minimum of $\varphi$ is not reached in $-1$ so that there exists an agent giving a better variance reduction than the antithetic method.
\end{remark}

\begin{theorem}
\label{thm:sde_coupling}
Let us consider two general SDEs
\begin{align*}
& dY^1_t = b_1(Y^1_t) dt + \sigma_1(Y^1_t) dW^1_t \\
& dY^2_t = b_2(Y^2_t) dt + \sigma_2(Y^2_t) dW^2_t.
\end{align*}
Then any coupling between $Y^1$ and $Y^2$ can be written through the coupling of the two Brownian motions $W^1$ and $W^2$ with
\begin{align}
dW^2_t =  \rho_t dW^1_t + (I_{d_2} - \rho_t \rho_t^\top)^{1/2} dW^3_t,
\end{align}
where $W^1$ and $W^3$ are independent and $\rho : [0,T] \to \mathcal{M}_{d_2}(\mathbb{R})$ is an adapted process such that for every $t$, $\rho_t \rho_t^\top \le I_{d_2}$; furthermore $\rho_t$ only depends on $(Y^1_t, Y^2_t, t)$.
\end{theorem}
\begin{proof}
We only prove for the corresponding Euler-Maruyama scheme. Coupling the two schemes $Y^1$ and $Y^2$ turns to choosing an adapted coupling of the Brownian increments $(\xi^1_1, \ldots, \xi^1_N)$ and $(\xi^2_1, \ldots, \xi^2_N)$ where $\xi^i \sim \mathcal{N}(0,I_{d_2})^{\otimes{N}}$. For every $0 \le k \le N$ and conditionally to $\mathcal{F}_{t_k}$, the law of $(Y^1_{t_{k+j}}, Y^2_{t_{k+j}})_{1 \le j \le N-k}$ only depends on $(Y^1_{t_k}, Y^2_{t_k})$ by the Markov property so $\rho_{t_k}$ only depends on $(Y^1_{t_k}, Y^2_{t_k}, t_k)$. Then two correlated centred Gaussian variables can be written through linear combinations of independent variables: $\xi^2_{k+1} = \rho_{t_k} \xi^1_{k+1} + \omega_{t_k} \xi^3_{k+1}$ where $\rho_{t_k}$ and $\omega_{t_k}$ are $d_2 \times d_2$ matrices and $\xi^3$ is independent of $\xi^1$. The one-variance condition on $\xi^2$ implies that $\omega_{t_k} = (I_{d_2} - \rho_{t_k}\rho_{t_k}^\top)^{1/2}$, implying also that $\rho_{t_k} \rho_{t_k}^\top \le I_{d_2}$.
\end{proof}

\begin{remark}[Link with Maximum Optimal Coupling]
\label{rk:MOT}
Our variance reduction problem consists in correlating two SDEs $X^1$ and $X^2$ with same SDE coefficients; following Theorem \ref{thm:sde_coupling} this boils down to finding an optimal correlation function $\rho(X^1_t, X^2_t, t)$. The minimization problem \eqref{eq:J_def} can be rewritten as the maximization of the $L^2$ transport cost, as
\begin{align*}
\mathbb{E}|f(X^1_T) - f(X^2_T)|^2 = 2\mathbb{E}|f(X_T)|^2 - 2\mathbb{E}[f(X^1_T) f(X^2_T)].
\end{align*}
\end{remark}



\end{document}